\documentclass[a4paper,10pt]{IEEEtran}
\usepackage[utf8]{inputenc}

\usepackage{color}
\usepackage{graphics}

\usepackage{amsmath,amssymb,amsfonts,amsthm,bm}

\usepackage{cite}

\usepackage[acronyms]{glossaries}

\usepackage[hidelinks]{hyperref}

\newtheorem{theorem}{Theorem}
\newtheorem{lemma}{Lemma}

\newacronym{sbl}{SBL}{sparse Bayesian learning}
\newacronym{fsbl}{F-SBL}{fast SBL}
\newacronym{em}{EM}{expectation-maximization}
\newacronym{ml}{ML}{maximum likelihood}
\newacronym{rvm}{RVM}{relevance vector machine}
\newacronym{ard}{ARD}{automatic relevance detection}
\newacronym{pdf}{PDF}{probability density function}

\newcommand{\ist}{\hspace*{.3mm}}

\newcommand{\iist}{\hspace*{1mm}}

\newcommand{\T}{\text{T}}

\begin{document}
\title{General Pruning Criteria for Fast SBL}

\author{\IEEEauthorblockN{Jakob M\"oderl$^\ast$, Erik Leitinger$^\ast$, Bernard Henri Fleury$^\dagger$} \\
    \IEEEauthorblockA{$^*$Graz University of Technology, Graz, Austria,  \{jakob.moederl, erik.leitinger\}@tugraz.at}\\
    \IEEEauthorblockA{$^\dagger$Technische Universität Wien, Wien, Austria, bernard.fleury@tuwien.ac.at}
    \thanks{We thank the Christian Doppler Research Association, the Austrian Federal Ministry for Digital and Economic Affairs, and the National Foundation for Research, Technology and Developmen for financially supporting this work within the framework of the Christian Doppler Laboratory of Location-Aware Electronic Systems. This research was in part funded by the Austrian Research Promotion Agency (FFG) under the project PRISM (project number: 999923584).}
}

\maketitle

\begin{abstract}
    \Gls{sbl} associates to each weight in the underlying linear model a hyperparameter by assuming that each weight is Gaussian distributed with zero mean and precision (inverse variance) equal to its associated hyperparameter. The method estimates the hyperparameters by marginalizing out the weights and performing (marginalized) \gls{ml} estimation. 
    \Gls{sbl} returns many hyperparameter estimates to diverge to infinity, effectively setting the estimates of the corresponding weights to zero (i.e., pruning the corresponding weights from the model) and thereby yielding a sparse  estimate of the weight vector.
    
    In this letter, we analyze the marginal likelihood as function of a single hyperparameter while keeping the others fixed, when the Gaussian assumptions on the noise samples and the weight distribution that underlies the derivation of \gls{sbl} are weakened.
    We derive sufficient conditions that lead, on the one hand, to finite hyperparameter estimates and, on the other, to infinite ones.
    Finally, we show that in the Gaussian case, the two conditions are complementary and coincide with the pruning condition of \gls{fsbl}, thereby providing additional insights into this algorithm.
\end{abstract}

\glsresetall

\section{Introduction}

Sparse signal reconstruction methods have attracted a lot of attention in the past 20 years \cite{baraniukSPM2007:CS,DonohoTIT2006:CS}.
\Gls{sbl} \cite{faulNIPS2001:AnalysisSBL} is such a method, which has found applications in many domains, e.g., line spectral estimation of acoustic and radio channels \cite{GreLeiWitFle:TWC2024,gerstoftSPL2016:SBLforDoA,moederlFusion2025:multi-dictionary-SBL}, joint channel estimation and decoding \cite{hansenTSP2018:IterativeReceiverSBL}, optical flow estimation \cite{dorazilICASSP2023}, blind deconvolution \cite{tzikasTIP2009:SBL-blind-deconvolution,wipfJMLR2014:blind-deconvolutions}, localization of neural current sources using MEG and EEG sensors \cite{seegerSPM2010:VBayesInference} and multi-pitch estimation of audio signals \cite{moederl2023TSP:structuredLSE}.

\Gls{sbl} computes a sparse estimate of a weight vector in a linear model corrupted by additive noise.
To do so, the prior of each weight is modeled as a zero-mean Gaussian distribution with an unknown precision (i.e., inverse variance) as hyperparameter \cite{faulNIPS2001:AnalysisSBL}.
The hyperparameters are estimated from the data by marginalizing out the weight vector and maximizing the marginal likelihood.
Then, an approximate posterior of the weights is computed from these hyperparameter estimates. 
In practice, iterative techniques, such as the expectation–maximization (EM) algorithm, are used to maximize the marginal likelihood.
However, a drawback of this approach is its slow convergence.
In \cite{faulNIPS2001:AnalysisSBL} the authors present an analysis of the marginal likelihood used in \gls{sbl} when a single hyperparameter is varying, while the others are kept fixed. This study provides a means to check whether the so-obtained section of the marginal likelihood has a single maximum or not.
Based on this analysis, \cite{tipping2003WAIS:FastMarginalSparseBayesian} presents a fast version of \gls{sbl}, coined \gls{fsbl}, that maximizes the marginal likelihood using coordinate ascent.
The maximization typically yields many hyperparameter estimates with infinite values, resulting in an estimated prior for the associated weights with zero variance. In other words, these weights are inferred to be zero with probability one and are thereby effectively ``pruned from the model'', producing a sparse estimate of the weight vector.

Readers interested in learning more about \gls{sbl} are referred to \cite{wipf2011TIP}, which provides a comprehensive treatment of Type-I and Type-II Bayesian estimation for sparse signal reconstruction, particularly highlighting the relationship between the two methods, as \gls{sbl} is rooted in this framework. Also noteworthy is \cite{palmerNIPS2005}, which presents the connection between EM-based implementations of both Type-I and Type-II methods.

\subsection*{Contribution}
In this letter, we analyze the marginal likelihood optimized in \gls{fsbl} with a focus on the pruning condition, i.e., on obtaining either finite or infinite hyperparameter estimates, when the Gaussian assumptions that underlie \gls{sbl} are weakened.
\begin{itemize}	
    \item Under these weaker assumptions, we derive two sufficient conditions leading to \gls{fsbl} to return finite hyperparameter estimates on the one hand and infinite ones on the other.
    
    \item We show that these two conditions are complementary and coincide with the pruning condition of \gls{fsbl} obtained under the Gaussian assumptions, see \cite{faulNIPS2001:AnalysisSBL,shutin2011TSP:fastVSBL,GreLeiWitFle:TWC2024,AmeGomICML2021:SBLStewpwiseRegression}. 

    \item In addition to the theoretical derivations, we provide a graphical interpretation that offers an intuitive perspective on the pruning condition, i.e., on the internal mechanism/functioning of \gls{fsbl}.

\end{itemize}

\section{Signal Model and Sparse Bayesian Learning}
\label{sec:signal-model}

In this Section, we introduce \gls{sbl} in its ``classical'' form, i.e., under the assumption that all involved distribution are Gaussian. However, we formulate the required equations in a general form so that they remain valid under the weaker Assumptions~A1--A4 stated in Section~\ref{sec:mathematical-analysis}, which provide the basis for the subsequent analysis.

We consider the linear model
\begin{equation}
    \bm{y} = \bm{A}\bm{x} + \bm{v}
    \label{eq:linear-model}
\end{equation}
where the observation $\bm{y}\in \mathbb{R}^N$ is a linear combination of $K<N$ out of $M$ (with $M \gg N$) columns of the overcomplete dictionary matrix $\bm{A}=[\bm{a}_1\cdots\bm{a}_M] \in\mathbb{R}^{N\times M}$ with sparse weight vector $\bm{x}=[x_1 \ist \cdots \ist x_M]^\T\in\mathbb{R}^M$ embedded in additive noise $\bm{v}$, resulting in the likelihood $p(\bm{y}|\bm{x})$.
We assume the true number of nonzero weights $K$ as well as the values of the weights to be unknown.

\Gls{sbl} assumes the random vector $\bm{v}$ to be Gaussian with (known or unknown) covariance and considers a parameterized prior \gls{pdf} of $\bm{x}$ of the form $p(\bm{x};\bm{\gamma})=\prod_{i=1}^{M}p(x_i;\gamma_i)$ with
$p(x_i;\gamma_i)=\mathrm{N}(x_i;\,0,\,\gamma_i^{-1})$, $i=1,\dots,M$ and $\bm{\gamma}=[\gamma_1 \ist\cdots\ist \gamma_M]^\T \in \mathbb{R}_{>0}^M$, where $\mathbb{R}_{>0}=\{x\in \mathbb{R}:x>0\}$.
The hyperparameter vector $\bm{\gamma}$ is estimated as the maximizer of the marginal likelihood
\begin{align}
    \label{eq:marginal-likelihood-objective}
    L(\bm{\gamma})
    \triangleq
    \int p(\bm{y}|\bm{x}) \prod_{i=1}^{M} p(x_i;\gamma_i)\,\mathrm{d}\bm{x}\ist,
\end{align}
i.e.,
\begin{align}
    \hat{\bm{\gamma}}_{\text{SBL}}
    =
    \underset{\bm{\gamma}\in\mathbb{R}_{>0}^M}{\arg\max} \iist L(\bm{\gamma})\ist.
    \label{eq:marginal-likelihood}
\end{align}
The estimate $\hat{\bm{\gamma}}_{\text{SBL}}$ is used to compute an approximation of the posterior of the weight vector $\bm{x}$ according to
\begin{align} \label{eq:weights-posterior}
    p(\bm{x}|\bm{y};\hat{\bm{\gamma}}_{\text{SBL}}) \propto p(\bm{y}|\bm{x})p(\bm{x};\hat{\bm{\gamma}}_{\text{SBL}}).
\end{align}

In practice, the maximization stage in \eqref{eq:marginal-likelihood} is traditionally performed by means of the EM-algorithm.
Alternatively, it can be carried out via coordinate ascent by explicitly expressing the dependency of \eqref{eq:marginal-likelihood-objective} on each single entry of $\bm{\gamma}$, while keeping the other entries fixed. Leveraging this dependency leads to \gls{fsbl} \cite{tipping2003WAIS:FastMarginalSparseBayesian}. 
Our investigation precisely relies on this dependency which was thoroughly studied in \cite{faulNIPS2001:AnalysisSBL}. We reproduce it here in a more general form that matches our purpose.

Given $i=1,\ldots,M$ let $\bm{\gamma}_{\sim i}$ denote the $M-1$ dim. vector obtained by removing the $i$th component $\gamma_i$ in $\bm{\gamma}$.
For such $i$, $\ell_i: \gamma_i \mapsto\ell_i(\gamma_i)=L(\bm{\gamma})$, with $\bm{\gamma}_{\sim i}$ considered fixed, is the $i$th section of the marginal likelihood $L(\bm{\gamma})$.
We easily show from \eqref{eq:marginal-likelihood-objective} that
\begin{align}\label{eq:marginal-likelihood-single-gamma}
    \ell_i(\gamma_i) = \int f_i(x_i;\bm{\gamma}_{\sim i})p(x_i;\gamma_i)\,\mathrm{d} x_i
\end{align}
where
\begin{align}\label{eq:partly-marginalized-likelihood}
    f_i(x_i;\bm{\gamma}_{\sim i}) = \int p(\bm{y}|\bm{x})\prod_{j=1,j\neq i}^{M} p(x_j;\gamma_j)\,\mathrm{d}\bm{x}_{\sim i}
\end{align}
is obtained by integrating out all weights $x_j$, $j\neq i$ in \eqref{eq:marginal-likelihood-objective} with $\bm{x}_{\sim i}$ defined similarly to $\bm{\gamma}_{\sim i}$.
Here, \eqref{eq:marginal-likelihood-single-gamma} can be viewed as the expectation 
of $x_i\mapsto f_i(x_i;\bm{\gamma}_{\sim i})$ when $x_i$ is random with \gls{pdf} $p(x_i;\gamma_i)$.

In classical \gls{sbl} the integral \eqref{eq:marginal-likelihood-objective} can be solved analytically, see Section~\ref{sec:classical-example} for details. 
It is shown in \cite{faulNIPS2001:AnalysisSBL} that in this case $\ell_i(\gamma_i)$ either exhibits a single global maximum, or continuously increases and converges to some finite value as $\gamma_i\uparrow \infty$.
This leads to the update rule
\begin{align}
    \hat{\gamma}_i &=
    \begin{cases}
        \displaystyle  \arg\max_{\gamma_i>0}\iist\ell_i(\gamma_i)
        &; \text{ if the maximum exists} \\
        \infty &; \text{ otherwise}
    \end{cases}
    \label{eq:fsbl-updates}
\end{align}
for the $i$th hyperparameter in the \gls{fsbl} algorithm.
The condition in \eqref{eq:fsbl-updates} can be checked analytically, see Section~\ref{sec:classical-example} or \cite{faulNIPS2001:AnalysisSBL} for details.
Note that as $\gamma_i\uparrow\infty$, the prior $p(x_i;\gamma_i)$ approaches a Dirac-delta distribution at zero, effectively forcing the estimate of the weight $x_i$ to zero according to \eqref{eq:weights-posterior}, and thus pruning the $i$th component $\bm{a}_i x_i$ in \eqref{eq:linear-model}.
We refer to \eqref{eq:fsbl-updates} as the pruning condition of \gls{fsbl}.

\section{Analysis of the Marginal Likelihood Under General Scale-Mixture Priors}
\label{sec:mathematical-analysis}

In this section we analyze the graph of $\ell_i(\gamma_i)$ in \eqref{eq:marginal-likelihood-single-gamma} with a focus on conditions which result in $\hat{\gamma}_i=\infty$ in \eqref{eq:fsbl-updates}, i.e., in the pruning of the $i$th component $\bm{a}_i x_i$. 
In this study, we relax the Gaussian assumption on the noise vector and the prior \glspl{pdf} $p(x_i;\gamma_i)$, $i=1,\ldots,M$ underlying classical \gls{sbl} and instead adopt the weaker assumptions listed below.
Note that under the latter all expressions \eqref{eq:marginal-likelihood-objective}--\eqref{eq:fsbl-updates} remain valid.

To alleviate the notation in this section we omit the index $i$ and write $f(x)=f(x_i;\hat{\bm{\gamma}}_{\sim i})$, $p(x;\gamma)=p(x_i;\gamma_i)$ and $\ell(\gamma)=\ell_i(\gamma)$.

The analysis relies on the following assumptions:
\begin{enumerate}
    \item[A1:] $p(x;\gamma)$ is even and belongs to a scale family of distributions with variance $\gamma^{-1}$, i.e., $p(x;\gamma)=p(-x;\gamma)=\gamma^{1/2} p(\gamma^{1/2}x;1)$, $\forall$ $\gamma \in \mathbb{R}_{ > 0}$ and $\int x^2 p(x;\gamma)\,\mathrm{d}x = \gamma^{-1}$.
    \item[A2:] The fourth moment of $p(x;1)$ is finite, i.e., $\int x^4 p(x;1)\,\mathrm{d}x < \infty$.
    	\footnote{
        It follows from Ljapounow's inequality that $p(x;1)$, and therefore $p(x;\gamma)$ for any $\gamma\in\mathbb{R}_{>0}$ by A2 have finite variance.}
    \item[A3:] $f(x)$ is four times continuously differentiable and $\lim_{\vert x\vert\uparrow\infty}f(x)=0$.
    \item[A4:] We assume that $f^{(4)}(x)$ fulfills the conditions that allow for interchanging limit and integral when we take the limit of \eqref{eq:theorem-divergence-proof-2} as $\gamma\uparrow\infty$, i.e.
    \begin{equation} \label{eq:lim_int_permutation}
        \hspace*{-5ex}
        \lim_{\gamma\uparrow\infty}\int\! R_3(x)p(x;\gamma)\,\mathrm{d}x= 
        \int\!\big(\lim_{\gamma\uparrow\infty}R_3(x)p(x;\gamma)\big)\,\mathrm{d}x
    \end{equation}
    where $R_3(x)$ is the remainder of the third-order Taylor polynomial of $f(x)$ at $0$, see \eqref{eq:Taylor-series}.
\end{enumerate}

Taylor's theorem \cite[Theorem 5.15]{rudin1976} states that we can write any function $g(x)$ that is $k$-times differentiable at $x=0$ as
\begin{equation}\label{eq:Taylor-series}
    g(x)=\sum_{n=0}^{k}\frac{g^{(n)}(0)}{n!}x^n + R_k(x)
\end{equation}
where $g^{(n)}(0)$ with $n=0,\ldots,k$ denotes the $n$th derivative of $g(x)$ at $0$. The sum in \eqref{eq:Taylor-series} is the $k$th order Taylor polynomial of $g(x)$ at $0$
and $R_k(x)$ is the remainder or error term that results from approximating $g(x)$ with this polynomial.

We first show the following lemma.
\begin{lemma}
    \label{lemma:marginal-likelihood-moments}
    Under Assumptions A1-A4 the function
    $\ell(\gamma)$, see \eqref{eq:marginal-likelihood-single-gamma}, can be written as
    \begin{equation}
        \label{eq:marginal-likelihood-moments}
        \ell(\gamma) = 
        \int_{-\infty}^{\infty}f(x)p(x;\gamma)\,\mathrm{d}x  = f(0) + \frac{f^{(2)}(0)}{2}\gamma^{-1} + o(\gamma^{-2})
    \end{equation}
    as $\gamma\uparrow\infty$. In particular, $\lim_{\gamma\uparrow\infty}\ell(\gamma)=f(0)$.
\end{lemma}
Here, $o(\cdot)$ stands for the little-$o$ notation.
\begin{proof}
    According to the first part of Assumption~A3, we can apply \eqref{eq:Taylor-series} with $k=3$ to $f(x)$. 
    Inserting into \eqref{eq:marginal-likelihood-single-gamma} yields
    \begin{align}\label{eq:expanding_l}
        \ell(\gamma) = \sum_{n=0}^{3} \frac{f^{(n)}(0)}{n!} \int x^n p(x;\gamma)\,\mathrm{d}x + \int R_3(x)p(x;\gamma)\,\mathrm{d}x
        .
    \end{align}
    Obviously, $\int x^{0} p(x;\gamma)\,\mathrm{d}x=1$. 
    Since $p(x;\gamma)$ is even, $\int x^n p(x;\gamma)\,\mathrm{d}x=0$ for all $n$ odd. Under Assumption~A1,
    $\int x^2 p(x;\gamma)\,\mathrm{d}x=\gamma^{-1}$. In Appendix~\ref{sec:appendix:lemma-proof} 
    we show that under rather weak sufficient conditions $\int R_3(x)p(x;\gamma)\,\mathrm{d}x=o(\gamma^{-2})$ as $\gamma\uparrow\infty$.
\end{proof}

\subsubsection*{Remark 1} Lemma~\ref{lemma:marginal-likelihood-moments} can be generalized in two directions by suitably relaxing and adjusting Assumptions~A1--A4. First, we can consider any arbitrary (i.e., non-even) zero-mean scale family $p(x;\gamma)$ with variance $\gamma^{-1}$ (by considering a Taylor polynomial of order two in \eqref{eq:Taylor-series}).
Second, we can consider the case where the first non-vanishing derivative of $f(x)$ at $x=0$ has order $2K$, i.e., $f^{(2n)}(0)=0$ for $n=1,\dots,K-1$ and $f^{(2K)}(0)>0$ (by using the $(2K+1)$th Taylor polynomial of $f(x)$ at $0$).
\medskip

The following theorem states a sufficient condition to obtain the estimate $\hat{\gamma}=\infty$ in \eqref{eq:fsbl-updates}.
To formulate it we apply \eqref{eq:Taylor-series} with $k=1$ to $f(x)$, yielding $f(x)=t(x) + R_1(x)$, where the first order Taylor polynomial $t(x)=f(0)+f^{(1)}(0)x$ is the tangent to $f(x)$ at $x=0$. Hence, $R_1(x)$ is the difference between $f(x)$ and its tangent $t(x)$ at $0$, as illustrated in Figure~\ref{fig:likelihood functions}. Finally, we define $\bar{R}_1(x)=R_1(x)+R_1(-x)$.

\begin{theorem}\label{theorem:divergence}
	The update rule \eqref{eq:fsbl-updates} returns
	$\hat{\gamma}=\infty$ if
	\begin{align}
		\label{eq:theorem-divergence}
		\bar{R}_1(x) < 0 \quad \text{for all}\ x \in \mathbb{R}_{>0}
        \,.
	\end{align}
\end{theorem}
\begin{proof}
Since the prior $p(x;\gamma)$ has mean zero, applying \eqref{eq:Taylor-series} with $k=1$ to $f(x)$ and inserting into \eqref{eq:marginal-likelihood-single-gamma} yields
\begin{align}\label{eq:theorem-divergence-proof-2}
	\ell(\gamma) = f(0) + \int_{-\infty}^{\infty}R_1(x)p(x;\gamma)\,\mathrm{d}x\ist .
\end{align}
The integral in \eqref{eq:theorem-divergence-proof-2} can be recast as $\int_{0}^{\infty} \bar{R}_1(x)p(x;\gamma)\,\mathrm{d}x$ since $p(x;\gamma)$ is even. Finally, since $p(x;\gamma)\geq 0$ it follows that $\ell(\gamma) < f(0)$ for all $\gamma \in \mathbb{R}_{>0}$ if \eqref{eq:theorem-divergence} holds. As a result, $\sup_{\gamma > 0} \ell(\gamma)=\lim_{\gamma\uparrow\infty}\ell(\gamma)=f(0)$, i.e., $\ell(\gamma)$ has no maximum.
\end{proof}

We also give a condition that is sufficient to obtain a finite estimate $\hat{\gamma}$ in \eqref{eq:fsbl-updates}.
\begin{theorem}
\label{theorem:convergence}
The update rule \eqref{eq:fsbl-updates} yields a finite estimate $\hat{\gamma}$ if
\begin{align}\label{eq:theorem-convergence}
f^{(2)}(0)&>0
\,.
\end{align}
\end{theorem}
\begin{proof}
    Under Assumption~A3, it follows by invoking Lebesgue's Dominated Convergence Theorem that $\lim_{\gamma\downarrow 0}\ell(\gamma)=f(0)P({0};1)< f (0)$ with $P(0;1)=\int_0^0 p(x;1)\,\mathrm{d}x<1$ denoting the probability of the random variable $x$ with \gls{pdf} $p(x;1)$ to be equal to $0$.
    Note that the last strict inequality follows from Assumption~A1. 
    From \eqref{eq:marginal-likelihood-moments} in Lemma~\ref{lemma:marginal-likelihood-moments} we have $\lim_{\gamma\uparrow\infty}\ell(\gamma)=f(0)$.
    Assuming $f^{(2)}(0) \neq 0$, the second term $\frac{f^{(2)}(0)}{2}\gamma^{-1}$
    in \eqref{eq:marginal-likelihood-moments} approaches zero at a slower rate than the third term $o(\gamma^{-2})$. Therefore, $\lim_{\gamma\uparrow\infty}\ell(\gamma)=f(0)$ is approached from above if in addition $f^{(2)}(0) > 0$. Since $\lim_{\gamma\downarrow0}\ell(\gamma) < \lim_{\gamma\uparrow\infty}\ell(\gamma)$ and $\ell(\gamma)$ is continuous by Assumption~A4, $\ell(\gamma)$ must exhibit a maximum in $\mathbb{R}_{>0}$ in this case. See Remark~1 for the extension of this result to the case $f^{(2)}(0)=0$.
\end{proof}

\subsubsection*{Remark 2} Theorems~\ref{theorem:divergence} and \ref{theorem:convergence} are mutually exclusive. This can be shown by assuming that \eqref{eq:theorem-convergence} in Theorem~\ref{theorem:convergence} is true, i.e., $f^{(2)}(0)=\delta>0$. Since $f^{(2)}(x)$ is continuous by Assumption~A4, there must exist some $\epsilon > 0$ such that $|f^{(2)}(x)-f^{(2)}(0)|<\delta$ for $x\in I_\epsilon=(-\epsilon,\epsilon)$, yielding $f^{(2)}(x)>0$ for $x\in I_\epsilon$. Therefore, if \eqref{eq:theorem-convergence} is true then $f^{(2)}(x)$ is strictly convex on $I_{\epsilon}$ and $R_{1}(x)\geq 0$ for $x \in I_{\epsilon}$ (with equality only if $x=0$), violating \eqref{eq:theorem-divergence} from Theorem~\ref{theorem:divergence}.

\section{Classical SBL Example and Graphical Illustration}
\label{sec:example-and-interpretation}

To demonstrate the usefulness of the theoretical results stated above, we apply them to classical \gls{sbl} and show their relation to the pruning condition of \gls{fsbl}. Additionally, we provide a graphical perspective that illustrates these results.

\subsection{Classical SBL Example}
\label{sec:classical-example}
Let us assume that the noise vector $\bm{v}$ in \eqref{eq:linear-model} is Gaussian, white, i.e., its components are independent, zero-mean Gaussian, with known precision $\lambda$. In this case, the likelihood function reads $p(\bm{y}|\bm{x})=\mathrm{N}(\bm{y};\bm{A}\bm{x},\lambda^{-1}\bm{I})$.
Inserting the (classical) \gls{sbl} prior, i.e., $p(x_i;\gamma_i)=\mathrm{N}(x_i;0,\gamma_i^{-1})$, $i=1,\dots,M$, into \eqref{eq:partly-marginalized-likelihood} and solving analytically we obtain
\begin{align}
f_i(x_i;\hat{\bm{\gamma}}_{\sim i})
\propto
\mathrm{N}(x_i; \mu_i, \sigma_i^2)
\, ,
\end{align}
where
$\mu_i = \sigma_i^{2} \lambda \bm{a}_i^\T \bm{M}_{\sim i}\bm{y}$,
$\sigma_i^2 = (\lambda \bm{a}_i^\T\bm{M}_{\sim i}\bm{a}_i)^{-1}$, with $\bm{M}_{\sim i}=\bm{I}-\lambda\bm{A}_{\sim i}(\lambda \bm{A}_{\sim i}^\T \bm{A}_{\sim i} + \mathrm{diag}(\hat{\bm{\gamma}}_{\sim i}))^{-1}\bm{A}_{\sim i}$. In this expression, $\bm{A}_{\sim i}$ denotes $\bm{A}$ with the $i$th column removed.

Taking the second derivative of $f_i(x_i;\hat{\bm{\gamma}}_{\sim i})\propto \mathrm{N}(x_i;\mu_i,\sigma_i^2)$, it can be shown that \eqref{eq:theorem-convergence} from Theorem~\ref{theorem:convergence} is equivalent to 
\begin{align}
|\mu_i| > \sigma_i
\label{eq:gaussian-convergence-criterion}
\end{align}
using some algebraic manipulations.
Hence, According to Theorem~\ref{theorem:convergence}, the update \eqref{eq:fsbl-updates} yields a finite estimate $\hat{\gamma}_i$ if \eqref{eq:gaussian-convergence-criterion} holds.
Furthermore, we show in Appendix~\ref{sec:appendix-theorem-strict-gaussian} that the inverse of \eqref{eq:gaussian-convergence-criterion}, i.e., $|\mu_i|\leq \sigma_i$, is equivalent to \eqref{eq:theorem-divergence} in Theorem~\ref{theorem:divergence}.
Thus, Condition \eqref{eq:gaussian-convergence-criterion} is necessary and sufficient for \eqref{eq:fsbl-updates} to return a finite estimate $\hat{\gamma}_i$.

Using some algebraic manipulations, it can be easily verified that Condition \eqref{eq:gaussian-convergence-criterion} is equivalent to the pruning condition of \gls{fsbl}, see \cite[Eq. (17)]{faulNIPS2001:AnalysisSBL}, \cite[Eq. (30)]{GreLeiWitFle:TWC2024} with $\kappa=1$, \cite[Eq. (11)]{shutin2011TSP:fastVSBL}, or \cite[Lemma 3.1]{AmeGomICML2021:SBLStewpwiseRegression}.
The parameters $\rho_i$ and $\zeta_i$ in \cite{GreLeiWitFle:TWC2024} coincide with $\mu_i$ and $\sigma_i^2$, respectively.
Hence, the threshold $\kappa$ introduced in \cite{GreLeiWitFle:TWC2024} can be interpreted as requiring the absolute value of the mean of $\mathrm{N}(x_i;\, \mu_i,\sigma_i^2)$ to be $\sqrt{\kappa}$ times the standard deviation for the respective weight to be nonzero.

\subsection{Graphical Perspective}
\label{sec:graphical-interpretation}

In this subsection, we again alleviate the notation by omitting the index $i$ and write $f(x)=f(x_i;\hat{\bm{\gamma}}_{\sim i})$, $p(x;\gamma)=p(x_i;\gamma_i)$, $\ell(\gamma)=\ell_i(\gamma)$, $\mu=\mu_i$, and $\sigma=\sigma_i$.

Figure~\ref{fig:likelihood functions} illustrates the function $f(x)\propto \mathrm{N}(x;\mu,\sigma^2)$ obtained in classical \gls{sbl} for Case (a) $|\mu|>\sigma$, and (b) $|\mu|<\sigma$. Specifically, Case (a) shows $f(x)\propto \mathrm{N}(x;1.5,1)$ and Case (b) shows $f(x)\propto \mathrm{N}(x;0.5,1)$.
In Case (a), We find that $f^{(2)}(x)>0$ for $x\in I_{0.5}=(-0.5,0.5)$, fulfilling \eqref{eq:theorem-convergence} in Theorem~\ref{theorem:convergence}. Furthermore, the function $f(x)$ is strictly convex on $I_{0.5}$, yielding $R_1(x)\geq 0$ for $x\in I_{0.5}$ violating \eqref{eq:theorem-divergence} in Theorem~\ref{theorem:divergence}.
In Case (b) we find $f^{(2)}(x)<0$ for $x\in I_{0.5}$ and $f(x)$ is concave on $I_{0.5}$, violating \eqref{eq:theorem-convergence} in Theorem~\ref{theorem:convergence}, while it is shown in Appendix~\ref{sec:appendix-theorem-strict-gaussian} that in this case \eqref{eq:theorem-divergence} in Theorem~\ref{theorem:divergence} is fulfilled.

\begin{figure}
\centering
\includegraphics{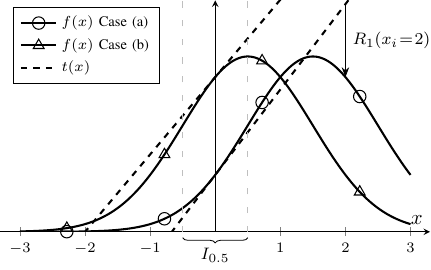}
\caption{Illustration of the graphs of $f(x)$ and its tangent $t(x)$ at $x=0$. In Case (a) $f(x)\propto\mathrm{N}(x;\,1.5,1)$ which is strictly convex on $I_{0.5}=(-0.5,0.5)$; in Case (b) $f(x)\propto\mathrm{N}(x;\, 0.5,1)$ which is strictly concave on $I_{0.5}$.
}
\label{fig:likelihood functions}
\end{figure}

To provide additional insights into Theorems~\ref{theorem:divergence} and \ref{theorem:convergence}, we further analyze the behavior of $\ell(\gamma)$ as $\gamma\uparrow\infty$ based on whether $f(x)$ is strictly convex/concave on a non-degenerate symmetric interval around zero, say $I_a=(-a,a)$ with $a>0$.
Note that the probability measure of $p(x;\gamma)$ outside the interval $I_a$ can be made arbitrarily small by choosing $\gamma$ sufficiently large.
Together with the fact that $f(x)p(x;\gamma)$ is integrable, see \eqref{eq:marginal-likelihood-single-gamma}, this allows us to restrict the following graphical discussion only to the interval $I_a$.
If $f(x)$ is strictly convex on $I_a$, such as in Case (a) for $I_{0.5}$, the tangent $t(x)$ is a lower bound on $f(x)$ and $\ell(\gamma)$ is lower bounded by $\int t(x)p(x;\gamma) \,\mathrm{d}x=f(0)$. Moreover, the difference $R_1(x)$ increases as $|x|$ increases due to the strict convexity of $f(x)$. Figure~\ref{fig:marignal-integral} illustrates $f(x)\propto \mathrm{N}(x;1.5,1)$ of Case (a) and the prior $p(x;\gamma)$ for two different values of $\gamma$ with $\gamma^{[1]}<\gamma^{[2]}$. A smaller value of $\gamma$ means that the prior $p(x;\gamma)$ is spread out over a larger region and, therefore, puts more weight on larger values of $|x|$ where $R_1(x)$ is larger. On the other hand, as $\gamma\uparrow\infty$, $p(x;\gamma)$ approaches a Dirac delta at zero and the value of the integral $\int R_1(x)p(x;\gamma)\,\mathrm{d}x$ approaches zero as well, yielding $\ell(\gamma)=f(0)$.
Hence, for $\gamma$ sufficiently large, $\ell(\gamma)$ is lower-bounded by $f(0)$ and reaches the lower bound $f(0)$ from above as $\gamma\uparrow\infty$. The converse is true if $f(x)$ is concave on $I_a$, such as in Case (b) for $I_{0.5}$. In this case, $t(x)$ is an upper bound for on $f(x)$ such that $\ell(\gamma)$ is increasing and reaches the bound $f(0)$ from below as $\gamma\uparrow\infty$.

\begin{figure}
\centering
\includegraphics{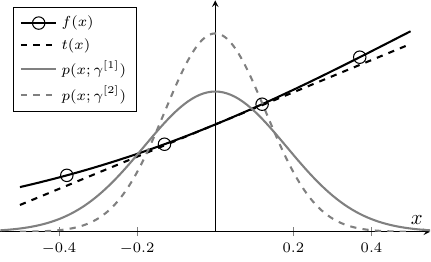}
\caption{Graph of the functions $f(x)$, $t(x)$, and $p(x;\gamma)$ for two values $\gamma^{[1]}$ and $\gamma^{[2]}$ of $\gamma$ such that $p(x;\gamma^{[1]})$ has most of its mass concentrated within the interval $I_{0.5}=(-0.5,0.5)$, $\gamma^{[1]} < \gamma^{[2]}$, and $f(x)$ is strictly convex on $I_{0.5}$.
}
\label{fig:marignal-integral}
\end{figure}

\section{Conclusion and Outlook}
We analyze the pruning condition in \gls{fsbl} when the Gaussian assumptions underlying the derivation of the algorithm are weakened to A1-A4 given in Section~\ref{sec:mathematical-analysis}.

By investigating the Taylor reminder $R_1(x)$, i.e., the difference between $f_i(x_i;\hat{\bm{\gamma}}_{\sim i})$ and its tangent at zero, we derive Theorem~\ref{theorem:divergence}, which states a sufficient condition for obtaining $\hat{\gamma}_i=\infty$ in \eqref{eq:fsbl-updates}, i.e., for pruning the component $\bm{a}_i x_i$ from the model \eqref{eq:linear-model}.
Our second main result, Theorem~\ref{theorem:convergence}, states a sufficient condition for obtaining $\hat{\gamma}_i<\infty$, i.e., for including the $i$th component $\bm{a}_i x_i$ in the model \eqref{eq:linear-model}. This condition is that the second derivative of $f_i(x_i;\hat{\bm{\gamma}}_{\sim i})$ is strictly positive at zero, i.e., $f_i^{(2)}(0;\hat{\bm{\gamma}}_{\sim i})>0$.
We show that in the Gaussian case, i.e., for ``classical'' \gls{sbl}, Conditions \eqref{eq:theorem-divergence} and \eqref{eq:theorem-convergence} in Theorems~\ref{theorem:divergence} and \ref{theorem:convergence} are complementary and that they are also equivalent to the pruning condition of \gls{fsbl} and, thereby, offer some more insight on the internal functioning on this algorithm.

\appendix
\subsection{Proof of Lemma~\ref{lemma:marginal-likelihood-moments}}
\label{sec:appendix:lemma-proof}

Let us assume that Assumptions A1--A3 are satisfied. We show that if \eqref{eq:lim_int_permutation} is true, then
\begin{equation}
h(\gamma)=\int R_3(x)p(x;\gamma)\,\mathrm{d}x = o(\gamma^{-2}) \,\, \text{as}\,\, \gamma\uparrow\infty .
\end{equation}
We start by expressing the remainder $R_3(x)$ in its Lagrange form:
\begin{equation}
\label{eq:appendix:remainder-largrange-form}
R_3(x) = 
\left\{
\begin{array}{ccl}
	0 & ; & x=0  \\
	\frac{f^{(4)}(\xi(x))}{4!}x^4 & ; & x\neq 0 
\end{array}
\right.
\end{equation}
where $\xi(x) \in (0,x)$ is a well-defined function with domain $\mathbb{R}\!\setminus\!\{0\}$.
Inserting \eqref{eq:appendix:remainder-largrange-form} into the integral in \eqref{eq:lim_int_permutation}  and performing the change of variables $u=\sqrt{\gamma}x$ we obtain
\begin{align}\label{eq:function-h}
h(\gamma)
&=
\gamma^{-2} \frac{1}{4!}\int_{\mathbb{R}\setminus\{0\}}
f^{(4)}\big(\xi\big(\gamma^{-1/2}u
\big)\big) u^{4}p(u;1)\,\mathrm{d}u .
\end{align}
For any fixed $u\in\mathbb{R}\!\setminus\!\{0\}$, $\lim_{\gamma\uparrow\infty}\xi\big(\gamma^{-1/2}u\big)= 0$, i.e. the function $\xi\big(\gamma^{-1/2}u\big)$ converges pointwise to $0$ on its domain $\mathbb{R}\!\setminus\!\{0\}$.  
By Assumption A3, as $\gamma \uparrow \infty$ the function $f^{(4)}\big(\xi\big(\gamma^{-1/2}u\big)\big)$ converges pointwise to the constant function $f^{(4)}(0)$ on its domain $\mathbb{R}\!\setminus\!\{0\}$. Invoking Assumption A3, we write
\begin{align*}
\lim_{\gamma\uparrow\infty}
\int_{\mathbb{R}\setminus\{0\}}
f^{(4)}\big(\xi\big(\gamma^{-1/2}u\big)\big) u^{4}p(u;1)\,\mathrm{d}u =
& \\
& \hspace*{-20ex}
f^{(4)}(0) \int_{\mathbb{R}\setminus\{0\}} u^{4}p(u;1)\,\mathrm{d}u .
\end{align*}
By Assumption A2, the right-hand integral is finite. Thus, from \eqref{eq:function-h}, we conclude that $h(\gamma)=o(\gamma^{-2})$ as $\gamma\uparrow\infty$.

{\em Comment:} 
A sufficient condition for the identity in \eqref{eq:lim_int_permutation} to hold is that $f^{(4)}(x)$ is uniformly integrable with respect to the finite measure with density $x^4p(x;1)$ relative to the Lebesgue measure on $\mathbb{R}$ \cite[Theorem II.6.4]{shiryayev1984}. 
This is the case if $\vert f^{(4)}(x)\vert$ is bounded by a function that is integrable with respect to the former measure, in which case Lebesgue's Dominated Convergence Theorem \cite[Theorem II.6.3]{shiryayev1984} can be invoked. Note that the finiteness of said measure is postulated in Assumption A2.

\subsection{Application of Theorem~\ref{theorem:divergence} to the Gaussian Case}
\label{sec:appendix-theorem-strict-gaussian}
Starting from a Gaussian \gls{pdf}
\begin{align}
f(x)=\mathrm{N}(x;\mu,\sigma^2)=\frac{1}{\sqrt{2\pi\sigma^2}}e^{-\frac{(x-\mu)^2}{2\sigma^2}}
\end{align}
we find
\begin{align}
f^\prime(x) &= \frac{\mu-x}{\sigma^2}\frac{1}{\sqrt{2\pi\sigma^2}}e^{-\frac{(x-\mu)^2}{2\sigma^2}} \\
t(x) &= \frac{1}{\sqrt{2\pi\sigma^2}}e^{-\frac{\mu^2}{2\sigma^2}}\Big[\frac{\mu x}{\sigma^2} + 1\Big]
\end{align}
and
\begin{align}
R_1(x) &= \frac{1}{\sqrt{2\pi\sigma^2}} e^{-\frac{\mu^2}{2\sigma^2}}\Big[ e^{-\frac{x^2-2\mu x}{2\sigma^2}} - \frac{\mu x}{\sigma^2} - 1\Big]
\end{align}
which yields
\begin{align}
\bar{R}_1(x)= \sqrt{\frac{2}{\pi\sigma^2}}e^{-\frac{\mu^2}{2\sigma^2}} \Big[ e^{-\frac{x^2}{2\sigma^2}} \cosh\Big(\frac{\mu x}{\sigma^2}\Big) -1 \Big]
\end{align}
after some algebraic manipulations. To determine the values of $\mu$ and $\sigma$ for which Equation~\eqref{eq:theorem-divergence} holds, we calculate
\begin{align}
\bar{R}_1^{\prime}(x) = \sqrt{\frac{2}{\pi\sigma^2}}\frac{1}{\sigma^2}e^{-\frac{\mu^2 + x^2}{2\sigma^2}} \big(k_1(x)-k_2(x)\big)
\end{align}
where we defined $k_1=\mu \sinh\big(\frac{\mu x}{\sigma^2}\big)$ and $k_2=x\cosh \big(\frac{\mu x}{\sigma^2}\big)$, and $(\cdot)^\prime$ denotes the derivative.
Analyzing the terms $k_1(x)$ and $k_2(x)$ for 
$x\in\mathbb{R}_{\geq0}$, 
we find that at $x=0$, both terms are zero, i.e., $x=0$ is a stationary point of $\bar{R}_1(x)$. Taking the derivatives of $k_1$ and $k_2$ yields
\begin{align}
k_1^\prime(x) &= \frac{\mu^2}{\sigma^2}\cosh \Big(\frac{\mu x}{\sigma^2}\Big) \\
k_2^\prime(x) &= \cosh \Big(\frac{\mu x}{\sigma^2}\Big) + \frac{\mu x}{\sigma^2} \sinh\Big(\frac{\mu x}{\sigma^2}\Big)
\, .
\end{align}
Therefore, if $\mu^2\leq\sigma^2$, then $k_2(x)$ grows faster than $k_1(x)$ for all $x\in\mathbb{R}_{>0}$ and $\bar{R}_1^\prime(x)<0$, i.e., $\bar{R}_1(x)$ is decreasing on $\mathbb{R}_{>0}$. Thus, $\bar{R}_1(0)=0$ is the global maximum of $\bar{R}_1(x)$ and \eqref{eq:theorem-divergence} is fulfilled.
On the other hand, if $\mu^2>\sigma^2$, then $\bar{R}_1(x)$ is increasing in  $(0,a)$ for some $a>0$, violating \eqref{eq:theorem-divergence}.

\bibliography{IEEEabrv,References_SBL_Graphical_Perspective}
\bibliographystyle{IEEEtran}

\end{document}